\documentclass[runningheads]{llncs}
\usepackage[T1]{fontenc}
\usepackage{comment}
\usepackage{amsfonts}
\usepackage{amsmath}
\usepackage{algorithm}
\usepackage{xcolor}
\usepackage{algorithmic}
\usepackage{algorithm}
\usepackage{algorithmic}
\usepackage{enumitem}
\usepackage{booktabs}
\DeclareMathOperator{\diag}{diag}

\DeclareMathOperator{\argmin}{argmin}

\newcommand{\R}{\mathbb{R}}
\newcommand{\K}{\mathcal{K}}
\newcommand{\one}{\mathbf{1}}
\newcommand{\proj}{\Pi}
\newcommand{\norm}[1]{\left\lVert #1 \right\rVert}
\newcommand{\ip}[2]{\left\langle #1,#2 \right\rangle}

\usepackage[hidelinks]{hyperref}

\newtheorem{assumption}{Assumption}
\usepackage{graphicx}
\begin{document}
\title{AdamX: Cosine similarity meets gradient descent}
%
%\titlerunning{Abbreviated paper title}
% If the paper title is too long for the running head, you can set
% an abbreviated paper title here
%

\author{Francisco Caldas\inst{1}\orcidID{0000-0001-5090-0216} \and
Ruben Belo\inst{1}\orcidID{0009-0006-8516-7732} \and
Cláudia Soares\inst{1}\orcidID{0000-0003-3071-6627}}
\authorrunning{F. Caldas et al.}
% First names are abbreviated in the running head.
% If there are more than two authors, 'et al.' is used.
%
\institute{NOVA School of Science of Technology \\ Universidade Nova de Lisboa, Caparica, Portugal \\
\email{f.caldas@campus.fct.unl.pt}}

%\author{Anonymous Author(s)}
%\institute{Anonymous Institution(s)}

%
\maketitle              % typeset the header of the contribution
\begin{abstract}
We introduce AdamX, a first-order optimizer that incorporates cosine similarity as an adaptive mechanism for controlling update magnitudes. The proposed method is scalable, model-agnostic, and straightforward to integrate into existing training pipelines. We further introduce a variance rectification scheme that promotes smoother optimization during the early stages of training. Overall, we provide empirical evidence that AdamX achieves competitive convergence rates across a range of benchmark datasets and architectures. Performance is evaluated in terms of the number of epochs required to reach predefined performance thresholds under a fixed hyperparameter budget. Code and Experiments available at: \url{https://github.com/FranciscoCaldas/adamX}

\end{abstract}
\section{Introduction}

%The area of optimization as been focused from a long time in designing algorithms to converge faster and consistently to a local minima \cite{theoria_gau_1822}, by following the direction of steepest descent \cite{cauchy1847methode}. Currently, many popular learning algorithms for optimizing non-convex objectives use some variant of stochastic gradient descent (SGD), in particularly following the adaptive stepsize scheme, such as ADAgrad \cite{ADAGRAD}, RMSprop \cite{rmsprop}, Adam \cite{adam}, and others. 
Gradient-based optimization is central to modern machine learning, where model training requires minimizing high-dimensional and generally non-convex objectives. While stochastic gradient descent (SGD) remains a foundational approach \cite{cauchy1847methode}, adaptive first-order methods have become widely used because they adjust update magnitudes according to observed gradient statistics. Prominent examples include AdaGrad~\cite{ADAGRAD}, RMSProp~\cite{rmsprop}, and Adam~\cite{adam}.

%In particular, Adam combines momentum with variance normalization, leading to fast and stable optimization across a wide range of deep learning tasks \cite{adam}. The empirical success of Adam comes primarely from the flexibility of the algorithm, and the adaptability of the optimizer to different deep learning tasks without heavy hyperparametrization.
Adam combines exponential moving averages of the gradient and its coordinate-wise squared magnitude to produce momentum-based, adaptively normalized updates~\cite{adam}. This combination has made Adam a practical default for many deep learning tasks, as it often provides stable training behavior with limited task-specific tuning. However, its adaptive normalization can also lead to problematic update dynamics and, in some settings, a failure to converge~\cite{reddi_amsgrad}.

%Recent works have focused on modifying the update dynamics of Adam and related methods, either through improved variance correction, decoupled regularization, or adaptive momentum formulations \cite{adamw,reddi_amsgrad,radam,adabelief,lion,jordan2024muon}. Nevertheless, most existing approaches rely exclusively on the magnitude of the gradient and its historical statistics, while largely ignoring the geometric relationship between consecutive optimization steps. In highly non-convex landscapes, where gradients may oscillate or become poorly aligned across iterations, exploiting directional consistency can provide valuable information about the local optimization geometry.
Existing Adam-type methods primarily improve the treatment of magnitude information in the optimization trajectory. AMSGrad enforces a monotone second-moment envelope to recover convergence guarantees in online convex optimization \cite{reddi_amsgrad}; AdamW decouples weight decay from adaptive updates \cite{adamw}; RAdam addresses instability in the early variance estimate \cite{radam}; and AdaBelief modifies the second-moment statistic to reflect deviation from the predicted gradient direction \cite{adabelief}. More recent optimizers, such as Lion \cite{lion} and Muon \cite{jordan2024muon}, further reconsider the form of the update rule or its preconditioning structure. Despite these developments, the alignment between successive gradients remains comparatively underused as a direct mechanism for modulating update magnitudes.

Directional alignment provides a computationally inexpensive signal for adapting update magnitudes. Consecutive gradients that point in similar directions indicate locally consistent optimization progress, whereas poorly aligned or opposing gradients may indicate oscillation or rapidly changing trajectory information. Cosine similarity captures this signal independently of gradient scale. Closely related to this motivation, GALA~\cite{jiang2025gala} adapts the learning rate using consecutive-gradient alignment together with a local curvature estimate, formulated through a one-dimensional online learning problem. AdamX instead incorporates alignment through a bounded multiplicative cosine controller within an Adam/AMSGrad-style coordinate-wise adaptive update. This design preserves the practical structure of adaptive moment methods while explicitly exploiting directional consistency.

%A different line of work is the use of second-order optimizers, that have shown improved results at the cost of computing an approximation of the Hessian matrix \cite{vyas2025soap,gupta2018shampoo}
Second-order and preconditioned optimization methods also exploit geometric information to improve training dynamics. Methods such as Shampoo and SOAP construct richer approximations to curvature or preconditioning structure, and can improve optimization performance in large-scale learning problems \cite{gupta2018shampoo,vyas2025soap}. In contrast, our objective is to investigate whether a lightweight scalar signal derived from consecutive gradient directions can provide useful geometric adaptivity while retaining the implementation simplicity and scalability of first-order Adam-type methods.

%Our work proposes evaluating in a comprehensive and reproducible manner the inclusion of cosine similarity as a method to manipulate the step-size of the update.

Our contributions are threefold. First, we introduce AdamX, an adaptive first-order optimizer that integrates a bounded cosine-similarity controller into an Adam-style moment-normalized update with a monotone second-moment envelope. 
Second, we provide an OCO analysis of a simplified momentum-free AdamX variant, showing that a bounded cosine controller can be incorporated into an adaptive projected-gradient scheme without worsening the standard convex regret rate.
Third, we evaluate AdamX across benchmark datasets and architectures, with default settings, measuring the number of epochs required to reach predefined test-performance thresholds.

\section{Method}

We consider stochastic optimization of an expected loss over parameters $\theta \in \mathbb{R}^d$. Let $\mathcal{D}$ denote the data distribution and let $B \sim \mathcal{D}$ be a randomly sampled mini-batch. The training objective is
\begin{equation}
    \min_{\theta \in \mathbb{R}^d}
    \; \mathbb{E}_{B \sim \mathcal{D}}
    \left[
        \mathcal{L}(\theta; B)
    \right].
\end{equation}
At iteration $t$, the optimizer observes the stochastic gradient
\begin{equation}
    g_t = \nabla_{\theta} \mathcal{L}(\theta_t; B_t).
\end{equation}
AdamX builds on the moment-normalized update used by Adam. Specifically, it maintains exponential moving averages of the stochastic gradient and its coordinate-wise square
\begin{equation}
    m_t = \beta_1 m_{t-1} + (1-\beta_1) g_t,
    \qquad
    v_t = \beta_2 v_{t-1} + (1-\beta_2) g_t^2,
\end{equation}
where $\beta_1,\beta_2 \in [0,1)$ and all squares are taken element-wise. The corresponding bias-corrected estimates are
\begin{equation}
    \hat{m}_t = \frac{m_t}{1-\beta_1^t},
    \qquad
    \hat{v}_t = \frac{v_t}{1-\beta_2^t}.
\end{equation}

AdamX addresses two complementary aspects of adaptive optimization. First, Adam may fail to converge even in simple convex settings because its adaptive denominator can produce unfavorable effective stepsizes. AMSGrad addresses this issue by retaining a coordinate-wise maximum of past second-moment estimates \cite{reddi_amsgrad}. Second, adaptive learning rates may exhibit high variance during the early stages of training, motivating rectification mechanisms such as RAdam~\cite{radam}. AdamX retains an AMSGrad-style monotone denominator and augments it with a bounded controller derived from directional agreement between consecutive gradients.

The distinctive component of AdamX is a cosine-similarity controller that modulates the magnitude of updates. For $t \geq 2$, we define
\begin{equation}
    c_t =
    \frac{
        \langle g_t, g_{t-1} \rangle
    }
    {
        \max\left\{
            \|g_t\|_2 \|g_{t-1}\|_2,\,
            \delta
        \right\}
    },
    \qquad
    \gamma_t = \exp(\lambda c_t),
\end{equation}
where $\lambda \geq 0$ controls the strength of the adaptation and $\delta > 0$ prevents division by zero. We set $\gamma_1 = 1$, since no preceding gradient is available at the first iteration. Because $c_t \in [-1,1]$, the controller is bounded as
\begin{equation}
    e^{-\lambda} \leq \gamma_t \leq e^{\lambda}.
\end{equation}
Thus, aligned consecutive gradients increase the effective update magnitude, while opposing gradients reduce it, without making the controller dependent on gradient scale. Also note that, for each parameter group $k$, $\gamma_t^{(k)}$ is computed from the cosine similarity between the current and previous gradients of the group.

Gradient alignment has previously been used to adapt learning rates in hypergradient-based methods \cite{almeida1998parameter,martinezrubio2017adam,gunes2018online}. AdamX uses this signal in a different optimizer structure: cosine similarity acts as a bounded multiplicative controller on top of an Adam-style moment-normalized update with a monotone second-moment envelope. Consequently, setting $\lambda=0$ removes the alignment controller and recovers the AMSGrad algorithm.

The normalization component of AdamX uses an AMSGrad-style variance envelope. We define
\begin{equation}
    \tilde{v}_t =
    \max\left\{
        \tilde{v}_{t-1},
        \hat{v}_t
    \right\},
\end{equation}
where the maximum is evaluated coordinate-wise. This construction ensures that the adaptive denominator is coordinate-wise non-decreasing, preventing increases in effective coordinate-wise stepsizes that arise solely from decreases in the second-moment estimate.

The resulting AdamX update combines moment normalization, the monotone variance envelope, and the cosine controller. Given a base learning rate $\eta>0$, the parameters are updated as
$
    \theta_{t}
    =
    \theta_{t-1}
    -
    \eta \gamma_t
    \frac{
        \hat{m}_t
    }{
        \sqrt{\tilde{v}_t} + \epsilon \mathbf{1}
    },
$
where all vector operations in the denominator are coordinate-wise.

\begin{algorithm}[ht]
\caption{AdamX Optimizer}
\label{alg:adamx}
\begin{algorithmic}[1]
\REQUIRE Learning rate $\eta$, decay rates $\beta_1, \beta_2 \in [0, 1)$, epsilon $\epsilon$
\REQUIRE Scaling parameter $\lambda$, initial parameters $\theta_0$
\STATE $m_0 \leftarrow 0$,$v_0 \leftarrow 0, \gamma_1 \leftarrow1$ 
\STATE $g_{\text{prev}} \leftarrow \mathbf{1}$ \COMMENT{Initialize previous gradient}
\FOR{$t = 1$ \TO $T$}
    \STATE $g_t \leftarrow \nabla_\theta f_t(\theta_{t-1})$ \COMMENT{Get gradients w.r.t. stochastic objective at $t$}
    \STATE $m_t \leftarrow \beta_1 m_{t-1} + (1-\beta_1) g_t$ 
    \STATE $v_t \leftarrow \beta_2 v_{t-1} + (1-\beta_2) g_t^2$ 
    \STATE $\hat{m}_t \leftarrow m_t / (1-\beta_1^t)$
    \STATE $\hat{v}_t \leftarrow v_t / (1-\beta_2^t)$ 
    \IF{$t>1$}
    {\color{red}\STATE $\gamma_t \leftarrow \exp\left( \lambda \cdot\text{cosinesimilarity}(g_t, g_{\text{prev}}) \right)$}
    \ENDIF
    \textcolor{red}{\STATE $\tilde{v}_t \leftarrow \max(\tilde{v}_t, \hat{v}_{t})$ \COMMENT{variance envelope}}
    \STATE $\theta_t \leftarrow \theta_{t-1} - \eta \cdot \frac{\gamma_t \hat{m}_t}{\sqrt{\tilde{v}_t} + \epsilon \mathbf{1}}$ \COMMENT{Update parameters}
    \STATE $g_{\text{prev}} \leftarrow g_t$ \COMMENT{Store gradient for next iteration}
\ENDFOR
\RETURN $\theta_t$
\end{algorithmic}
\end{algorithm}
Our regret analysis considers a modified AdamX update designed for online convex optimization. In particular, the analyzed variant removes momentum, uses the current subgradient as the update direction, projects onto a convex feasible set in an adaptive diagonal metric, and controls the alignment-scaled learning rate through a non-increasing envelope. These modifications isolate the effect of the bounded cosine controller while enabling a standard adaptive online-learning analysis.

\section{Regret Guarantees for AdamX-OCO}

\paragraph{Scope of the analysis.} The practical AdamX optimizer in Algorithm~\ref{alg:adamx} uses momentum and applies the raw alignment multiplier $\gamma_t$ to the update magnitude. To obtain a transparent regret guarantee, we analyze an OCO variant that removes momentum, projects onto a convex feasible set using an adaptive diagonal metric, and replaces the raw alignment-scaled step size with a non-increasing envelope. This variant isolates the effect of the bounded cosine controller while retaining the AMSGrad-style monotone second-moment envelope.
\begin{algorithm}[t]
\caption{AdamX for OCO}
\label{alg:adamx-lite}
\begin{algorithmic}[1]
\REQUIRE Convex compact set $\K\subset\R^d$, sequence $\alpha_t>0$, parameters $\beta_2\in[0,1)$, $\epsilon>0$, $\lambda\ge 0$, $\rho\in[0,1]$
\REQUIRE Initial point $\theta_1\in\K$
\STATE $v_0\leftarrow 0$, $\bar v_0\leftarrow 0$, $q_0\leftarrow +\infty$, $g_0\leftarrow \bot, \gamma_1 \leftarrow 1$
\FOR{$t=1$ \TO $T$}
    \STATE Play $\theta_t$ and observe $g_t\in\partial f_t(\theta_t)$
    \STATE $v_t\leftarrow \beta_2v_{t-1}+(1-\beta_2)g_t^2$
    \STATE $\hat v_t\leftarrow v_t/(1-\beta_2^t)$
    {\color{red}\STATE $\tilde v_t\leftarrow \max\{\tilde v_{t-1},\hat v_t\}$ }
    \IF{$t>1$}
        {\color{red}\STATE $c_t\leftarrow \ip{g_t}{g_{t-1}}/(\norm{g_t}_2\norm{g_{t-1}}_2)$}
        {\color{red}\STATE $\gamma_t\leftarrow \exp(\lambda c_t)$ \COMMENT{AdamX}}
    \ENDIF
    {\color{blue}\STATE $q_t\leftarrow \min\{q_{t-1},\alpha_t\gamma_t\}$ \COMMENT{OCO Adaptation}}
    \STATE $H_t\leftarrow \diag(\sqrt{\tilde v_t}+\epsilon\one)$
    \STATE $\theta_{t+1}\leftarrow \argmin_{\theta\in\K}\norm{\theta-(\theta_t-q_tH_t^{-1}g_t)}_{H_t}^2$
\ENDFOR
\RETURN $\theta_{T+1}$
\end{algorithmic}
\end{algorithm}
\paragraph{Adaptive projected-gradient interpretation.} Define the weighted norm $\norm{x}_{H_t}^2=x^\top H_tx$ and the effective metric 
\begin{equation} 
A_t=\frac{H_t}{q_t}. 
\end{equation} 
Since multiplication of the projection metric by a positive scalar does not change the projection, Algorithm~\ref{alg:adamx-lite} can equivalently be written as \begin{equation} 
\theta_{t+1} = \argmin_{\theta\in\K} \left\{ \ip{g_t}{\theta} + \frac{1}{2}\norm{\theta-\theta_t}_{A_t}^2 \right\}. 
\end{equation} 
The monotone envelope $q_t$ is introduced solely for analysis: together with the monotone second-moment envelope, it ensures that $A_t$ is coordinate-wise nondecreasing.

\paragraph*{Convex Regret Guarantee} \paragraph{Online convex optimization setting.} At round $t$, the learner chooses $\theta_t\in\K$, observes a convex loss $f_t:\K\to\R$, and receives a subgradient $g_t\in\partial f_t(\theta_t)$. For any comparator $u\in\K$, the regret is \begin{equation} 
R_T(u) = \sum_{t=1}^T \left( f_t(\theta_t)-f_t(u) \right). \end{equation} 
By convexity, 
\begin{equation} 
R_T(u) \le \sum_{t=1}^T \ip{g_t}{\theta_t-u}. 
\end{equation}
\begin{assumption}[Bounded domain and gradients] \label{assump:bounded} 
There exist constants $D_\infty>0$ and $G_\infty>0$ such that, for every $\theta,u\in\K$ and every $t$, 
\begin{equation} 
\norm{\theta-u}_\infty\le D_\infty, \qquad \norm{g_t}_\infty\le G_\infty. 
\end{equation} 
\end{assumption}

\paragraph{Bounded alignment controller.} The stabilized cosine similarity satisfies $c_t\in[-1,1]$, and therefore the AdamX alignment multiplier obeys 
\begin{equation} 
e^{-\lambda}\le \gamma_t=\exp(\lambda c_t)\le e^\lambda 
\end{equation} 
for every $t$ (with $\gamma_1=1$ by definition).

\begin{lemma}[One-step adaptive projected-gradient bound] \label{lem:one-step} 
For every $u\in\K$, 
\begin{equation} 
\ip{g_t}{\theta_t-u} \le \frac{1}{2} \left( \norm{\theta_t-u}_{A_t}^2 - \norm{\theta_{t+1}-u}_{A_t}^2 \right) + \frac{1}{2}\norm{g_t}_{A_t^{-1}}^2. \end{equation} 
Equivalently, 
\begin{equation} 
\ip{g_t}{\theta_t-u} \le \frac{1}{2q_t} \left( \norm{\theta_t-u}_{H_t}^2 - \norm{\theta_{t+1}-u}_{H_t}^2 \right) + \frac{q_t}{2}\norm{g_t}_{H_t^{-1}}^2. 
\end{equation} 
\end{lemma}

\begin{proof} The optimality condition of the projected update gives 
$$
\ip{g_t+A_t(\theta_{t+1}-\theta_t)}{u-\theta_{t+1}} \ge 0. 
$$
Combining this inequality with 
\begin{equation} 
2\ip{\theta_t-\theta_{t+1}}{A_t(\theta_t-u)} = \norm{\theta_t-\theta_{t+1}}_{A_t}^2 + \norm{\theta_t-u}_{A_t}^2 - \norm{\theta_{t+1}-u}_{A_t}^2.
\end{equation} 
 Applying Young's inequality to $\ip{g_t}{\theta_t-\theta_{t+1}}$ yields the result. 
\end{proof}

\paragraph*{Convex Regret Bound for AdamX-OCO}

\begin{theorem}[Convex OCO regret]
Suppose Assumption~\ref{assump:bounded} holds.
Run Algorithm~\ref{alg:adamx-lite} with $\alpha_t=\eta/\sqrt{t}$ for some $\eta>0$. Then for every $u\in\K$,
\begin{equation}
    R_T(u)
\le
\frac{D_\infty^2}{2q_T}\sum_{i=1}^d H_{T,i}
+
\frac{1}{2}\sum_{t=1}^T q_t\sum_{i=1}^d\frac{g_{t,i}^2}{H_{t,i}}.
\end{equation}
Moreover, using the coarse bounds $\epsilon\le H_{t,i}\le G_\infty+\epsilon$,
\begin{equation}
R_T(u)
\le
\frac{dD_\infty^2(G_\infty+\epsilon)}{2\eta e^{-\lambda}}\sqrt{T}
+
\frac{\eta e^\lambda dG_\infty^2}{\epsilon}\sqrt{T}.
\end{equation}
Consequently,
$$
R_T(u)=O(\sqrt{T}).
$$
\end{theorem}

\begin{proof}
By convexity,
$$
R_T(u)\le \sum_{t=1}^T\ip{g_t}{\theta_t-u}.
$$
Applying Lemma 1 and summing over $t$ gives
$$
R_T(u)
\le
\frac{1}{2}\sum_{t=1}^T
\left(
\norm{\theta_t-u}_{A_t}^2-\norm{\theta_{t+1}-u}_{A_t}^2\right)
+
\frac{1}{2}\sum_{t=1}^T\norm{g_t}_{A_t^{-1}}^2.
$$
Because $\tilde v_t$ is coordinatewise nondecreasing and $q_t$ is nonincreasing, the matrix sequence $A_t=H_t/q_t$ is positive semidefinite nondecreasing. Therefore the first sum telescopes with an additional nonnegative metric-growth term and can be bounded as
$$
\frac{1}{2}\sum_{t=1}^T
\left(
\norm{\theta_t-u}_{A_t}^2-\norm{\theta_{t+1}-u}_{A_t}^2\right)
\le
\frac{1}{2}\norm{\theta_1-u}_{A_1}^2
+
\frac{1}{2}\sum_{t=2}^T\norm{\theta_t-u}_{A_t-A_{t-1}}^2.
$$
Since $A_t$ is diagonal and $\norm{\theta_t-u}_\infty\le D_\infty$, this is at most
$$
\frac{D_\infty^2}{2}\sum_{i=1}^d A_{T,i}
=
\frac{D_\infty^2}{2q_T}\sum_{i=1}^d H_{T,i}.
$$
The second term is
$$
\frac{1}{2}\sum_{t=1}^T\norm{g_t}_{A_t^{-1}}^2
=
\frac{1}{2}\sum_{t=1}^T q_t\sum_{i=1}^d\frac{g_{t,i}^2}{H_{t,i}}.
$$
This proves the first bound.

Since $\gamma_t\in[e^{-\lambda},e^\lambda]$ and $\alpha_t=\eta/\sqrt{t}$, the envelope satisfies
$$
q_T\ge \frac{\eta e^{-\lambda}}{\sqrt{T}},
\qquad
q_t\le \frac{\eta e^\lambda}{\sqrt{t}}.
$$
Also $H_{t,i}\ge \epsilon$ and, under $\norm{g_t}_\infty\le G_\infty$, the second-moment and scalar rectification terms are bounded so that $H_{t,i}\le G_\infty+\epsilon$. Hence
$$
\frac{D_\infty^2}{2q_T}\sum_{i=1}^d H_{T,i}
\le
\frac{dD_\infty^2(G_\infty+\epsilon)}{2\eta e^{-\lambda}}\sqrt{T}.
$$
For the second term,
$$
\frac{1}{2}\sum_{t=1}^T q_t\sum_{i=1}^d\frac{g_{t,i}^2}{H_{t,i}}
\le
\frac{1}{2}\sum_{t=1}^T \frac{\eta e^\lambda}{\sqrt{t}}\frac{dG_\infty^2}{\epsilon}
\le
\frac{\eta e^\lambda dG_\infty^2}{\epsilon}\sqrt{T}.
$$
Combining the two inequalities gives the stated result.
\end{proof}

\begin{remark}[Effect of the alignment parameter]
The regret rate is unchanged by the alignment factor, but the constants scale with $e^\lambda$. This is expected: the multiplier $\gamma_t$ is bounded between $e^{-\lambda}$ and $e^\lambda$. %A large value of $\lambda$ allows large variations in the effective step size, which is undesirable in worst-case OCO analysis. This supports using small $\lambda$, clipping $\gamma_t$, or warming up the alignment controller.
\end{remark}

\paragraph*{Using Raw Alignment Instead of the Envelope}

The monotone envelope $q_t=\min\{q_{t-1},\alpha_t\gamma_t\}$ is theoretically convenient, but it removes part of the intended behavior of AdamX: when gradients become strongly aligned, the method cannot re-increase the effective step size if the envelope has already decreased.

If one instead uses the raw effective step size
$$
r_t=\alpha_t\gamma_t,
$$
and defines
$$
A_t=\frac{H_t}{r_t},
$$
then $A_t$ need not be monotone, even when $H_t$ is monotone. The proof still yields a data-dependent variation bound.

\begin{proposition}[Variation-dependent regret with raw alignment]
Consider the AdamX-OCO update with $r_t=\alpha_t\gamma_t$ instead of the monotone envelope $q_t$. Then for every $u\in\K$,
\begin{equation}
R_T(u)
\le
\frac{1}{2}\norm{\theta_1-u}_{A_1}^2
+
\frac{1}{2}\sum_{t=2}^T\norm{\theta_t-u}_{(A_t-A_{t-1})_+}^2
+
\frac{1}{2}\sum_{t=1}^T\norm{g_t}_{A_t^{-1}}^2,
\end{equation}
where $(A_t-A_{t-1})_+$ denotes the positive part of the symmetric matrix $A_t-A_{t-1}$.
\end{proposition}

This statement is more faithful to the practical optimizer. It says that AdamX-OCO keeps sublinear regret when the metric variation induced by the denominator and the alignment factor is controlled. In adversarial sequences, however, the cosine signal can oscillate, and the variation term may be large.

\paragraph*{Momentum and the Full AdamX Algorithm}

The full AdamX update uses $\hat m_t$ rather than $g_t$. In OCO, convexity gives
$$
f_t(\theta_t)-f_t(u)\le \ip{g_t}{\theta_t-u},
$$
whereas the projected update controls a term involving $\hat m_t$. Consequently,
$$
\ip{g_t}{\theta_t-u}
=
\ip{\hat m_t}{\theta_t-u}
+
\ip{g_t-\hat m_t}{\theta_t-u}.
$$
The first term can be handled by adaptive mirror descent. The second is a momentum-bias term. A direct bound gives
\begin{equation}
    \sum_{t=1}^T\ip{g_t-\hat m_t}{\theta_t-u}
\le
D_\infty\sum_{t=1}^T\norm{g_t-\hat m_t}_1,
\end{equation}
which can be linear in $T$ for adversarial gradient sequences \cite{defossez2022a}.

\iffalse
\subsection{Strongly Convex Losses}

If each $f_t$ is $\mu$-strongly convex, one may aim for logarithmic regret. In that case the base step size should scale as $\alpha_t=\eta/t$, not $\eta/\sqrt{t}$. The adaptive metric should be monotone and the effective step size should be controlled. A choice is
$$
q_t=\min_{1\le k\le t}\frac{\eta\gamma_k}{k}.
$$
Under bounded gradients, compactness, monotone adaptive metrics, and a suitable relation between $\eta$ and $\mu$, one can derive $O(\log T)$ regret by combining the standard strongly convex OGD argument with the adaptive-metric analysis above. The constants again deteriorate with $e^\lambda$.

\fi
\section{Experiments}

We evaluate the proposed algorithm against comparable optimizers using an experimental protocol inspired by DeepOBS~\cite{schneider2018deepobs} and AlgoPerf~\cite{Dahl2023AlgoPerf}. When designing empirical evaluations for deep learning optimizers, we focus on three key aspects.
\textbf{(1) Generalization.} The goal of optimization in deep learning is to learn models that generalize well to unseen data. Although some prior studies focus primarily on training metrics, improvements in training loss do not necessarily translate into better test performance. Accordingly, our evaluation emphasizes test-set performance throughout.
\textbf{(2) Stochasticity.} The observed performance can vary substantially due to random initialization. To mitigate the influence of these sources of randomness and ensure fair comparisons, all optimizers are evaluated using the same set of five random seeds, and results are reported as averages across runs.
\textbf{(3) Realistic evaluation setting.} Optimizer performance is highly dependent on the model architecture and dataset. Consequently, we adopt established benchmark architectures from DeepOBS~\cite{schneider2018deepobs} together with widely used datasets, ensuring evaluation on representative and commonly studied tasks. For consistency, we evaluate all optimizers, including ours, with the default hyperparameters \cite{pmlr-v139-schmidt21a}. 

Following this principles, the main evaluation tool is the number of epochs necessary to achieve a predetermined test set accuracy. Unlike AlgoPerf\cite{Dahl2023AlgoPerf}, which is more focused on algorithmic speed, we do not evaluate on wall-clock runtime, which has well-known drawbacks, such as dependency on
hardware or weak reproducibility. By evaluating on epochs, we evaluate performance against the number of gradient evaluations,
which typically dominates the total computational costs. 
\subsubsection{Baseline Algorithms}
To evaluate AdamX, we compare against widely used first-order optimizers spanning adaptive, momentum-based, and non-adaptive methods. Specifically, we consider \textbf{SGD} \cite{sgd}; \textbf{Adagrad}, which accumulates squared historical gradients \cite{ADAGRAD}; \textbf{RMSProp}, which replaces Adagrad's cumulative statistic with an exponential moving average \cite{rmsprop}; \textbf{Adam} \cite{adam}; \textbf{AdamW}, which decouples weight decay from adaptive updates \cite{adamw}; \textbf{AMSGrad}, which enforces a non-decreasing second-moment estimate \cite{reddi_amsgrad}; \textbf{RAdam}, which introduces variance rectification during early training \cite{radam}; \textbf{Yogi}, which controls excessive growth of the variance estimate \cite{yogi}; \textbf{Lion}, which updates parameters using the sign of the momentum vector \cite{lion}; and \textbf{Adan}, which incorporates Nesterov-style momentum into adaptive moment estimation \cite{xie2024adan}.

\begin{figure}[t]
    \centering

    \begin{minipage}{0.48\linewidth}
        \centering
        \includegraphics[width=\linewidth]{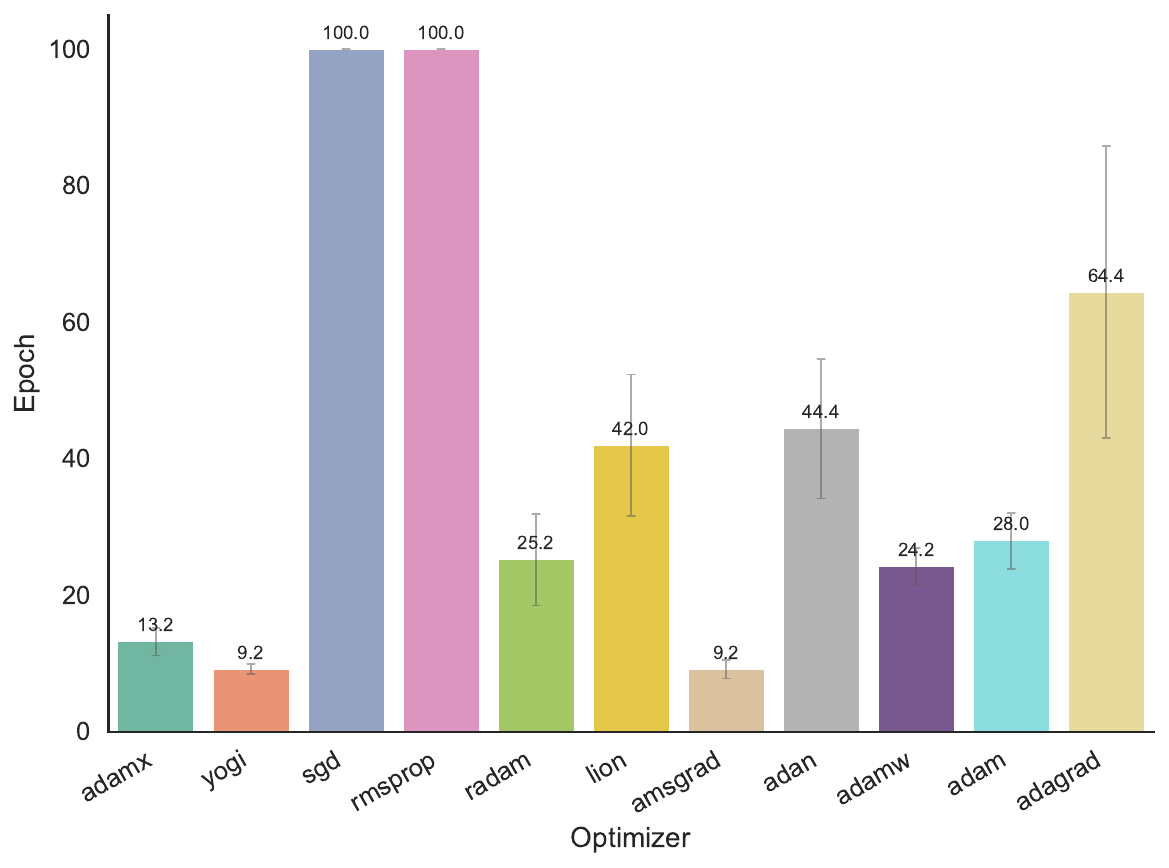}
        \caption{Number of epochs to reach the test accuracy target. Each optimizer is evaluated over five seeds. Yogi, AMSGrad, and AdamX achieve the best performance, while SGD and RMSProp fail to reach the target.}
        \label{fig:bar_plot}
    \end{minipage}
    \hfill
    \begin{minipage}{0.48\linewidth}
        \centering
        \includegraphics[width=\linewidth]{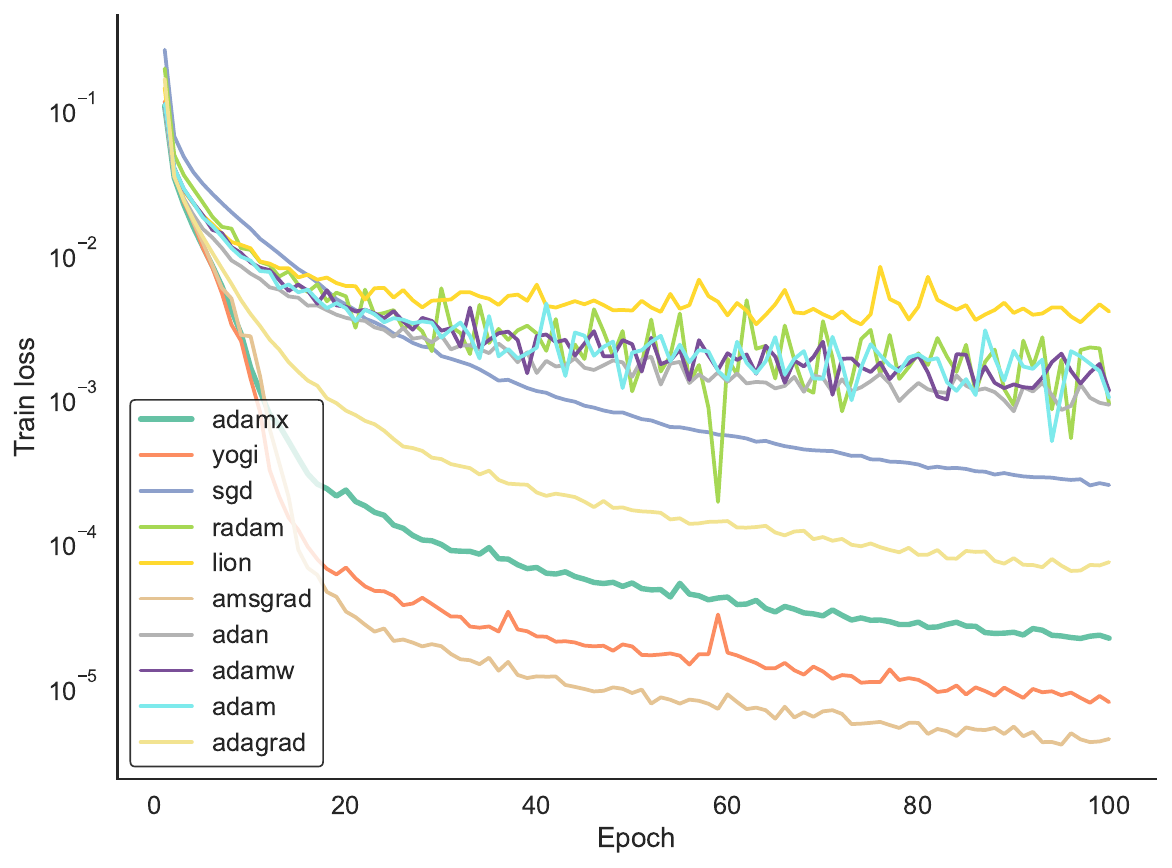}
        \caption{Training loss over 100 epochs on MNIST. Methods with better generalization (AMSGrad, Yogi, AdamX) also exhibit lower training loss. RMSProp is omitted due to significantly higher loss values. }
        \label{fig:train_loss1}
    \end{minipage}

\end{figure}
\subsubsection{MNIST}
On MNIST, we use a three-layer CNN with default hyperparameters and measure the epochs required to reach $0.994$ test accuracy, up to $100$ epochs. Figure~\ref{fig:bar_plot} shows that AdamX is competitive with the best-performing optimizers, AMSGrad and Yogi. SGD does not reach the target, RMSprop fails in all runs due to gradient collapse, and Adagrad shows the largest variance across seeds. The training-loss curves in Figure~\ref{fig:train_loss1} are consistent with these results, with AMSGrad, AdamX, and Yogi among the fastest methods to reach the target.

\subsubsection{CIFAR-10}
\begin{figure}[t]
    \centering

    \begin{minipage}{0.48\linewidth}
        \centering
        \includegraphics[width=\linewidth]{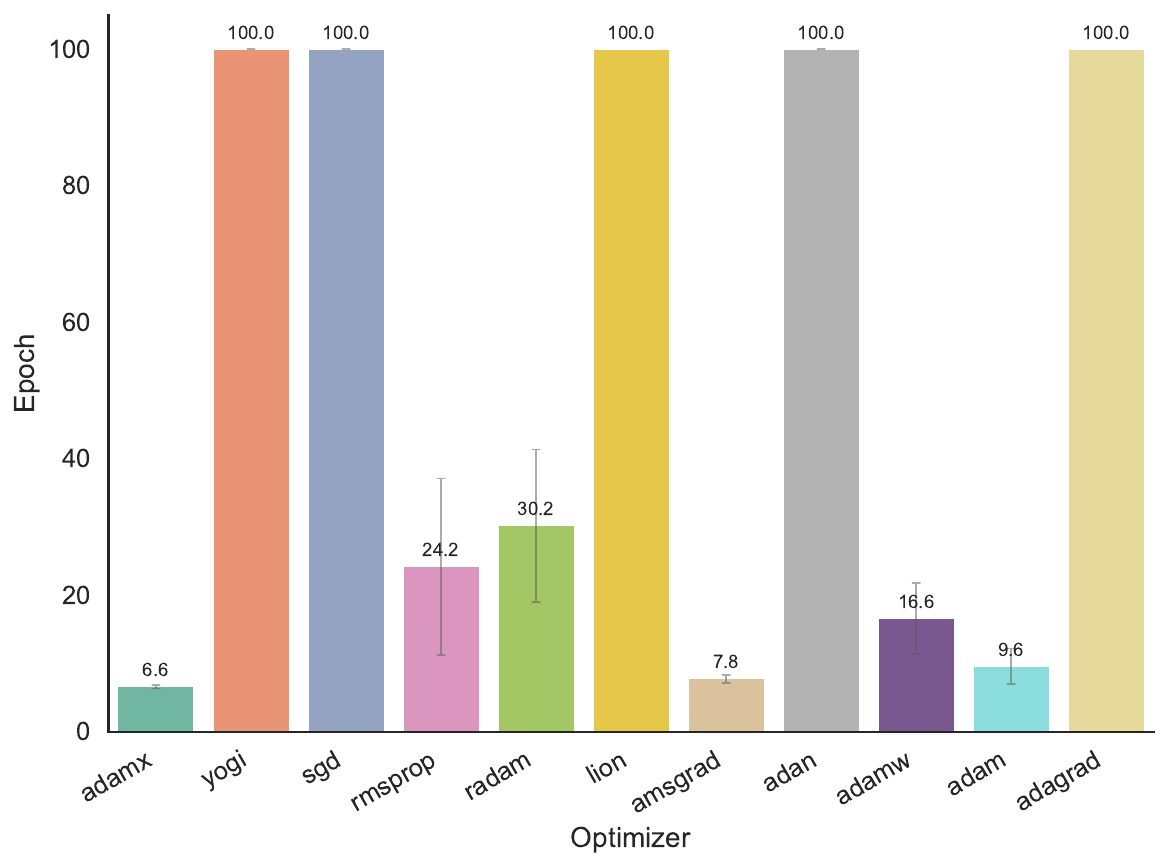}
        \caption{Number of epochs to reach the test accuracy target. Each optimizer is evaluated over five seeds. AdamX obtains the lowest mean number of epochs to reach the target, with similar values obtained by AMSGrad and Adam.}
        \label{fig:bar_plot2}
    \end{minipage}
    \hfill
    \begin{minipage}{0.48\linewidth}
        \centering
        \includegraphics[width=\linewidth]{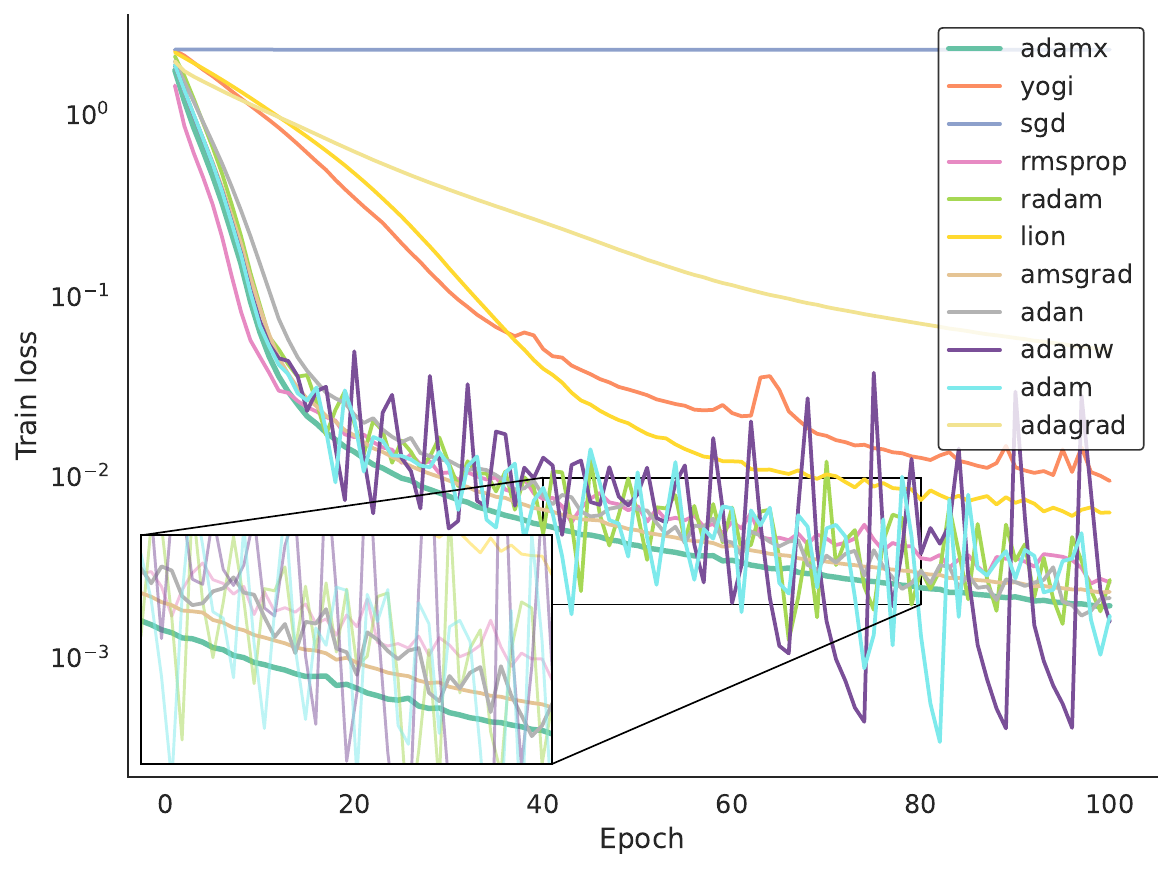}
        \caption{Training loss over 100 epochs on CIFAR-10. Methods with better generalization (AMSGrad, Yogi, AdamX) also exhibit lower training loss. RMSProp is omitted due to significantly higher loss values.}
        \label{fig:train_loss3}
    \end{minipage}

\end{figure}

CIFAR-10 is more challenging than MNIST; to focus on optimizer behavior, we use the fixed CifarNet architecture~\cite{jordan2024muon}. The target test accuracy is $0.84$, with a maximum of $100$ epochs.

Figure~\ref{fig:bar_plot2} shows that five out of eleven optimizers fail to reach the target, indicating that the threshold captures a demanding training regime. AdamX, AMSGrad, and Adam are the strongest methods, with AdamX achieving the lowest mean number of epochs and outperforming the closely related AMSGrad baseline. Figure~\ref{fig:train_loss3} further shows that AdamX, AMSGrad, and Adan exhibit smoother training-loss trajectories, whereas RAdam, Adam, and AdamW display larger oscillations.

Table~\ref{tab:epochs_target} summarizes the results. Overall, AdamX and AMSGrad require the fewest gradient evaluations to reach the target, with AdamX comparing favorably on CIFAR-10. The results also illustrate that lower training loss does not necessarily imply better generalization; for example, Adagrad obtains a low MNIST training loss but requires more epochs to reach the test-accuracy target.

\begin{table}[!ht]
\caption{Number of epochs until target accuracy, and training loss at 100 Epochs. Lower is better. Maximum number of runs is 100. \textbf{Best}, \underline{second-best}, and \textit{third-best}
results are highlighted.}
\label{tab:epochs_target}
\centering
\begin{tabular}{lcccc}
\toprule
&  \multicolumn{2}{c}{MNIST} &  \multicolumn{2}{c}{CIFAR-10} \\
\midrule
\textbf{Optimizer} &
\textbf{Epochs ($\downarrow$)} &
\textbf{Train Loss ($\downarrow$) ($\times 10^{-5}$)}  &\textbf{Epochs ($\downarrow$)} &
\textbf{Train Loss($\downarrow$)}\\
\midrule
AdamX (Ours)    &  \textit{13.2 $\pm$ 2.00} &  \textit{2.17 $\pm$ 0.41} &   \textbf{6.6 $\pm$ 0.54} & \textit{1.96e-03} \\
\midrule
Yogi       & \textbf{9.2 $\pm$ 0.734}& \underline{0.79 $\pm$ 0.13 } &100& 9.59e-03\\
SGD  & 100 &  25.00 $\pm$ 0.52& 100& 2.30\\
RMSProp   &  100 & 1609.97 $\pm$ 295.46  & 24.2 $\pm$ 28.94& 2.64e-03\\
Radam      & 25.2 $\pm$6.67 & 90.39 $\pm$ 13.29  &30.2 $\pm$24.93&2.72e-03\\
Lion     &  42.0 $\pm$10.36 & 397.40 $\pm$ 4.07  & 100.0&6.42e-03\\
AMSGrad   &  \textbf{9.2 $\pm$1.35}&  \textbf{0.44 $\pm$ 0.02 }& \underline{7.8 $\pm$ 1.30}&2.34e-03\\
Adan     &  44.4 $\pm$ 10.19 &  89.99 $\pm$ 12.00 & 100&2.16e-03\\
AdamW    &  24.2 $\pm$ 2.73 &  112.03 $\pm$ 67.39 &16.6 $\pm$ 11.65&\textbf{1.60e-03}\\
Adam   & 28.0 $\pm$ 4.09 &  100.33 $\pm$35.05 & \textit{9.6 $\pm$ 5.86}&\underline{1.72e-03}\\
Adagrad     &  64.4 $\pm$21.39 &  7.32 $\pm$ 1.645& 100&5.08e-02\\

\bottomrule
\end{tabular}
\end{table}

\section{Conclusions}

We presented AdamX, an adaptive first-order optimizer that augments Adam/AMSGrad-style updates with a bounded cosine similarity controller. A simplified OCO analysis shows that, under a monotone envelope on the cosine-scaled step size, the alignment mechanism is compatible with standard adaptive regret guarantees. Empirically, AdamX is competitive with ten established optimizers and achieves the best result in the considered CIFAR-10 setting. These results indicate that directional alignment is a promising lightweight source of adaptivity, motivating future work on hyperparameter robustness, second-order extensions, and larger-scale training regimes.

\iffalse 
We have presented AdamX, an adaptive first-order optimization algorithm that incorporates a bounded cosine-similarity controller into an Adam-style update rule. 
%
Our theoretical analysis studies a simplified OCO variant and shows that, when the cosine-scaled step size is controlled by a monotone envelope, the proposed alignment mechanism is compatible with standard adaptive online-learning regret guarantees.
%
 We performed empirical studies, where AdamX demonstrated competitive performance across benchmark datasets against other 10 well established algorithms. In particular, AdamX was ranked first in the CIFAR10 dataset. These results suggest that directional alignment can serve as a useful source of adaptivity when combined with already known algorithms, such as Adam or AMSgrad. In future work, we intend to further analyze the robustness of the method to hyperparameter tuning and adapt the same method to second-order optimizers, such as SOAP. Also, as deep learning steers towards more complex architectures and training regimes, a valuable direction is to study this algorithm for Large Language Model (LLM) and Diffusion Model (DM) pretraining. 
\fi

\newpage
\begin{credits}
\subsubsection{\ackname} This work was partially supported by NOVA LINCS (UID/04516) funded by FCT IP, and the Neuraspace AI Fights Space Debris project (C626449889-00463050), co-funded by Recovery and Resilience Plan and NextGeneration
EU Funds, www.recuperarportugal.gov.pt. The authors have no competing interests to declare that are
relevant to the content of this article.

\vspace{4pt}

\includegraphics[width=0.5\textwidth]{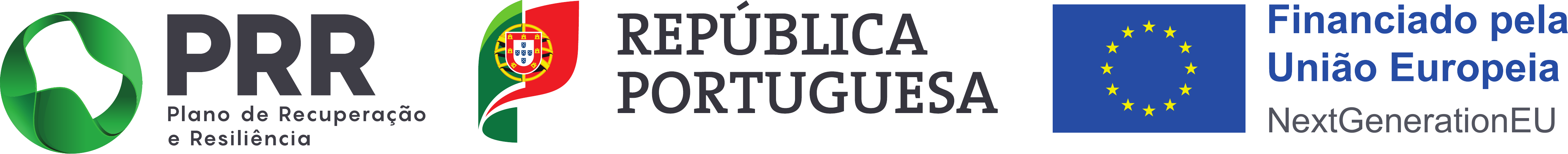}
\end{credits}

%
% ---- Bibliography ----
%
% BibTeX users should specify bibliography style 'splncs04'.
% References will then be sorted and formatted in the correct style.
%
 \bibliographystyle{splncs04}
 \bibliography{fixedbib}

@article{rmsprop,
  title = {Lecture 6.5-rmsprop: Divide the gradient by a running average of its recent magnitude},
  journal = {COURSERA: Neural networks machine learning},
  volume = {4},
  number = {2},
  pages = {26},
  year = {2012},
}

@inproceedings{vyas2025soap,
  title = {{SOAP}: Improving and stabilizing shampoo using adam for language modeling},
  author = {N. Vyas and D. Morwani and R. Zhao and I. Shapira and D. Brandfonbrener and L. Janson and S. Kakade},
  booktitle = {ICLR},
  year = {2025},
}

@inproceedings{gupta2018shampoo,
  title = {Shampoo: Preconditioned stochastic tensor optimization},
  author = {V. Gupta and T. Koren and Y. Singer},
  booktitle = {ICML},
  pages = {1842--1850},
  year = {2018},
  organization = {PMLR},
}

@article{cauchy1847methode,
  author = {A. L. Cauchy},
  title = {M{\'e}thode g{\'e}n{\'e}rale pour la r{\'e}solution des syst{\`e}mes d'{\'e}quations simultan{\'e}es},
  journal = {Comptes Rendus Hebd Seances Acad Sci},
  volume = {25},
  pages = {536--538},
  year = {1847},
}

@incollection{almeida1998parameter,
  author = {L. B. Almeida and T. Langlois and J. D. Amaral and A. Plakhov},
  title = {Parameter Adaptation in Stochastic Optimization},
  booktitle = {On-Line Learning Neural Networks},
  publisher = {CUP},
  year = {1998},
}

@mastersthesis{martinezrubio2017adam,
  author = {D. M. Rubio},
  title = {Convergence Analysis of an Adaptive Method of Gradient Descent},
  school = {U. Oxf.},
  type = {MSc thesis},
  year = {2017}
}

@inproceedings{adamw,
  title = {Decoupled Weight Decay Regularization},
  author = {I. Loshchilov and F. Hutter},
  booktitle = {ICLR},
  year = {2019},
}

@article{jiang2025gala,
  title   = {Online Learning-guided Learning Rate Adaptation via Gradient Alignment},
  author  = {R. Jiang and A. Kavis and A. Mokhtari},
  journal = {arXiv preprint arXiv:2506.08419},
  year    = {2025}
  }

@inproceedings{
gunes2018online,
title={Online Learning Rate Adaptation with Hypergradient Descent},
author={A.G. Baydin and R. Cornish and D.M. Rubio and M. Schmidt and F. Wood},
booktitle={ICLR},
year={2018}
}

@article{ADAGRAD,
  author = {J. Duchi and E. Hazan and Y. Singer},
  title = {Adaptive Subgradient Methods for Online Learning and Stochastic Optimization},
  journal = {JMLR},
  volume = {12},
  pages = {2121--2159},
  year = {2011},
}

@InProceedings{pmlr-v139-schmidt21a,
  title = 	 {Descending through a Crowded Valley - Benchmarking Deep Learning Optimizers},
  author =       {Schmidt, Robin M and Schneider, Frank and Hennig, Philipp},
year = 	 {2021},
booktitle={ICML}
}

@inproceedings{adam,
  author = {D. P. Kingma and J. Ba},
  title = {Adam: A Method for Stochastic Optimization},
  booktitle = {ICLR},
  year = {2015},
}

@article{defossez2022a,
  title = {A Simple Convergence Proof of {Adam} and {Adagrad}},
  author = {A. D{\'e}fossez and L. Bottou and F. Bach and N. Usunier},
  journal = {TMLR},
  issn = {2835-8856},
  year = {2022},
  note = {},
}

@inproceedings{schneider2018deepobs,
  title = {Deep{OBS}: A Deep Learning Optimizer Benchmark Suite},
  author = {F. Schneider and L. Balles and P. Hennig},
  booktitle = {ICLR},
  year = {2019},
}

@article{Dahl2023AlgoPerf,
  title = {{Benchmarking Neural Network Training Algorithms}},
  author = {G. E. Dahl and F. Schneider and Z. Nado and N. Agarwal and C. S. Sastry and P. Hennig and S. Medapati and R. Eschenhagen and P. Kasimbeg and D. Suo and J. Bae and J. Gilmer and A. L. Peirson and B. Khan and R. Anil and M. Rabbat and S. Krishnan and D. Snider and E. Amid and K. Chen and C. J. Maddison and R. Vasudev and M. Badura and A. Garg and P. Mattson},
  year = {2023},
  archiveprefix = {arXiv},
  eprint = {2306.07179},
  journal={arXiv preprint arXiv:2306.07179},
}

@article{sgd,
  title = {A Stochastic Approximation Method},
  author = {H. Robbins and S. Monro},
  journal = {Annals Mathematical Statistics},
  volume = {22},
  number = {3},
  pages = {400--407},
  year = {1951},
}

@inproceedings{yogi,
  author = {M. Zaheer and S. Reddi and D. Sachan and S. Kale and S. Kumar},
  booktitle = {NeurIPS},
  title = {Adaptive Methods for Nonconvex Optimization},
  volume = {31},
  year = {2018},
}

@article{xie2024adan,
  title = {Adan: Adaptive Nesterov momentum algorithm for faster optimizing deep models},
  author = {X. Xie and P. Zhou and H. Li and Z. Lin and S. Yan},
  journal = {IEEE TPAMI},
  year = {2024},
}

@misc{jordan2024muon,
  title = {Muon: An optimizer for hidden layers in neural networks},
  author = {K. Jordan and Y. Jin and V. Boza and Y. Jiacheng and F. Cesista and L. Newhouse and J. Bernstein},
  volume = {6},
  number = {3},
  pages = {4},
  year = {2024},
}

@inproceedings{reddi_amsgrad,
  title = {On the Convergence of Adam and Beyond},
  author = {S. J. Reddi and S. Kale and S. Kumar},
  booktitle = {ICLR},
  year = {2018},
}

@inproceedings{lion,
  title = {Symbolic Discovery of Optimization Algorithms},
  author = {X. Chen and C. Liang and Da Huang and E. Real and K. Wang and H. Pham and X. Dong and T. Luong and C. J. Hsieh and Y. Lu and Q. V. Le},
  booktitle = {NeurIPS},
  year = {2023},
}

@inproceedings{radam,
  title = {On the Variance of the Adaptive Learning Rate and Beyond},
  author = {L. Liu and H. Jiang and P. He and W. Chen and X. Liu and J. Gao and J. Han},
  booktitle = {ICLR},
  year = {2020},
}

@article{adabelief,
  title = {Adabelief optimizer: Adapting stepsizes by the belief in observed gradients},
  author = {J. Zhuang and T. Tang and Y. Ding and S. C. Tatikonda and N. Dvornek and X. Papademetris and J. Duncan},
  journal = {NeurIPS},
  year = {2020},
}
 %\bibliography{mybibliography}

\end{document}